\documentclass[sigconf]{acmart}
\usepackage{booktabs} 
\setcopyright{rightsretained}

\usepackage{multirow}
\usepackage{booktabs} 
 \usepackage{csquotes}
 \usepackage{amsmath}
 \usepackage{amssymb}
 \usepackage{setspace}
 \usepackage{graphicx} 
 \usepackage{csquotes}
 \usepackage{mathtools}
 \usepackage{algorithm}
\usepackage[noend]{algorithmic}
\usepackage{caption}

\newcommand{\oneoneiahype}{$(1+1)\text{IA}^\text{Hyp}$ }

\newcommand{\oneonerlsone}{$(1+1)~\text{RLS}_1$ }

\newcommand{\oneoneIA}{$(1+1)$~IA$^{Hyp}$~}

\newcommand{\linHD}{M$_{\text{linHD}}$~}
\newcommand{\expoF}{M$_{\text{expoF(x)}}$~}
\newcommand{\expoHD}{M$_{\text{expoHD}}$~}

\newcommand\T{\rule{0pt}{2.6ex}}      
\newcommand\B{\rule[-1.2ex]{0pt}{0pt}}

\acmDOI{10.1145/nnnnnnn.nnnnnnn}

\acmISBN{978-x-xxxx-xxxx-x/YY/MM}

\acmConference[GECCO '19]{the Genetic and Evolutionary Computation Conference 2019}{July 13--17, 2019}{Prague, Czech Republic}
\acmYear{2019}
\copyrightyear{2019}

\acmPrice{15.00}

\begin{document}
\title{On Inversely Proportional Hypermutations with Mutation Potential}

\author{Dogan Corus}
\affiliation{
  \institution{Department of Computer Science\\ University of Sheffield}
  \city{Sheffield} 
  \postcode{S1 4DP, UK}
}
\email{d.corus@sheffield.ac.uk}

\author{Pietro S. Oliveto}
\affiliation{
  \institution{Department of Computer Science\\ University of Sheffield}
  \city{Sheffield} 
  \postcode{S1 4DP, UK}
}
\email{p.oliveto@sheffield.ac.uk}

\author{Donya Yazdani}
\affiliation{
  \institution{Department of Computer Science\\ University of Sheffield}
  \city{Sheffield} 
  \postcode{S1 4DP, UK}
}
\email{dyazdani1@sheffield.ac.uk}

\begin{abstract}
Artificial Immune Systems (AIS) employing hypermutations with linear static mutation potential have recently been shown to be very effective at escaping local optima of combinatorial optimisation problems
at the expense of being slower during the exploitation phase compared to standard evolutionary algorithms.
In this paper we prove that considerable speed-ups in the exploitation phase may be achieved 
with dynamic inversely proportional mutation potentials (IPM)
and argue that the potential should decrease inversely to the distance to the optimum rather than to the difference in fitness. Afterwards we define a simple (1+1)~Opt-IA, that uses IPM hypermutations and ageing, for realistic applications where optimal solutions are unknown. 
%
The aim of the AIS is to approximate the ideal behaviour of the inversely proportional hypermutations better and better as the search space is explored.
We prove that such desired behaviour, and related speed-ups, occur for a well-studied bimodal benchmark function called \textsc{TwoMax}.
Furthermore, we prove that the (1+1)~Opt-IA with IPM efficiently optimises a third bimodal function, \textsc{Cliff}, by escaping its local optima while Opt-IA with static potential cannot, thus requires exponential expected runtime in 
the distance between the cliff and the optimum.
\end{abstract}


\maketitle

\section{Introduction}
Artificial Immune Systems (AIS) for optimisation are generally inspired by the clonal selection principle  \cite{Burnet1959}.
For this reason they are also often referred  to as {\it clonal selection} algorithms \cite{Brownlee2007}. 
In the literature two key features of clonal selection algorithms have been identified \cite{DeCastroTimmisBook}: 
\begin{enumerate}
\item  The proliferation rate of each immune cell is proportional to its affinity with the selective antigen: the higher the affinity, 
the higher the number of offspring generated (clonal selection and expansion). 
\item The mutation suffered by each immune cell during reproduction is inversely proportional to the affinity of the cell receptor with the antigen: the higher the affinity, 
the smaller the mutation (affinity maturation).
\end{enumerate}
Indeed well-known clonal selection algorithms employ mutation operators applied to immune cells (i.e., candidate solutions) with a rate that decreases with their similarity to the 
antigen (i.e., global optima) during affinity maturation. Often such operators are referred to as {\it inversely proportional hypermutations} (IPH). 
Popular examples of such clonal selection algorithms are Clonalg \citep{DecastroVonzuben2002}   and Opt-IA \citep{CutelloPavoneTimmis2006}. 

The ideal behaviour of the IPH operator is that the mutation rate is minimal in proximity of the global optimum and should increase as the difference between the fitness of the global optimum
and that of the candidate solution increases. However, achieving such behaviour in practice may be problematic because the fitness of the global optimum is usually unknown.
As a result, in practical applications information about the problem is used to identify bounds on (or estimates of)  the value of the global optimum\footnote{Alternatively, the fitness of the best candidate solution is sometimes used and 
the mutation rate of the rest of the population is inversely proportional to the best.}.
Thus, the closer is the estimate to the actual value of the global optimum, the closer should the behaviour of the IPH operator be to the desired one.
On the other hand,  if the bound is much higher (e.g., for a maximisation problem)  than the true value, then there is a risk that the mutation rate is too high in proximity of the global optimum 
i.e., the algorithm will struggle to identify the optimum. 

Previous theoretical analyses, though, have highlighted various problems with 
IPH operators even when the fitness value of the global optimum is known.
Zarges analysed the effects of mutating candidate solutions with a rate that is inversely proportional to their fitness for the OneMax problem  \cite{Zarges2008}.
She considered two different rates for the decrease of the mutation rate as the fitness increases: a linear decay (i.e., each bit flips with probability $\textsc{OneMax(x)}/\text{Opt}$ where $Opt$ is the optimum value) 
and an exponential decay (i.e., each bit flips with probability $e^{- \rho \frac{\textsc{OneMax(x)}}{\text{Opt}}}$ where $\rho$ is called the {\it decay} parameter).
The motivation behind these choices are that the former operator flips in expectation exactly the number of bits that maximises the probability of reaching the optimum in the next step while the latter is the
operator used in the Clonalg algorithm.  
She showed that if the optimum of OneMax is known, then an algorithm employing such a mutation operator will require exponential time to optimise OneMax with overwhelming probability (w.o.p.)
in both cases of linear or exponential decays of the mutation rate.
The reason is that the initial random solutions that have roughly half the fitness (i.e., $n/2$) of the global optimum (i.e., $n$) have very high mutation rates. Such rates do not allow the algorithm to make any progress with any reasonable probability even if the decay parameter $\rho$ is chosen very carefully (i.e., $\rho = \ln n$ leads to reasonable mutation rates between $1/n$ and $1/2$ for every
possible fitness value of OneMax).
Hence, the desired behaviour of the mutation operator, achieved by assuming the optimum is known, leads to very inefficient algorithms even for the simple OneMax function. 

On the other hand, the Opt-IA clonal selection algorithm, that also uses very high mutation rates (called hypermutations with mutation potential), has been proven to be efficient by employing a selection strategy called {\it stop at first constructive mutation} (FCM). The strategy evaluates fitness after each bit is flipped and interrupts the hypermutation immediately if an improvement is detected.
Indeed, the operator using a static mutation potential (i.e., a linear number of bits are flipped unless an improvement is found along the way) has been proven to be very effective at escaping local
optima of standard multimodal benchmark functions \cite{CorusOlivetoYazdani2017} and at finding arbitrarily good approximations for the Number Partitioning NP-Hard problem \cite{CorusOlivetoYazdani2018} at the expense
of being slower at hillclimbing during exploitation phases (i.e., it is up to a linear factor slower than standard bit mutation  for easy functions such as OneMax and LeadingOnes).
Hence, differently from other clonal selection algorithms (such as Clonalg and Opt-IA without FCM), hypermutations with mutation potential coupled with FCM can cope with the desired behaviour of inversely proportional mutation rates. 

In this paper we consider whether Opt-IA may become faster during exploitation phases if inversely proportional mutation potentials are applied rather than static ones.
The reason to believe this is that the hypermutations waste many fitness function evaluations during exploitation because
as the optimum is approached, the probability of flipping bits that lead to an improvement decreases. 
Hence, with high probability the operator flips wrong bits at the beginning of a hypermutation and ends up wasting a linear number
of fitness function evaluations. On the other hand, if the mutation potential decreases as the optimum is approached, 
then the amount of wasted evaluations should decrease accordingly.

A previous runtime analysis for OneMax of the inversely proportional hypermutations with mutation potential used by Opt-IA has shown that the mutation rate is always in the range $[(c/2)n, cn]$ where $M=cn$ is the highest mutation potential. Hence it does not decrease inversely proportional to the optimum as desired~\cite{JansenZarges2011}. 

In this paper we first show that considerable speed-ups in exploitation phases may be achieved, compared to static hypermutations, 
if mutation rates decrease appropriately with either the fitness or the distance to the optimum. 
To show this we analyse the IPH operators for standard unimodal benchmark functions where the performance of static mutation potentials is well-understood~\cite{CorusOlivetoYazdani2017}.
Through the analysis we show what speed-ups may be hoped for by analysing the operators in the ideal situation where the location of the optimum is known.
A result of this first analysis is that a mutation potential that increases exponentially with the Hamming distance to the optimum is the most promising out of three considered 
inverse potentials since it provides the larger speed-ups. Furthermore, using the Hamming distance rather than fitness as measure to quantify proximity to the optimum
makes the operator robust to fitness function scaling.
Hence we consider this operator, called  \expoHD, in the rest of the paper.

Afterwards we propose a clonal selection algorithm that we call (1+1)~Opt-IA~\footnote{The well-known Opt-IA AIS uses both operators in combination~\citep{CutelloPavoneTimmis2006}.}
that uses IPH and ageing to be applied in practical applications where the optimum is not known. 
The algorithm uses the best solution it has encountered to estimate the mutation rates.
In the literature it has been shown that ageing allows algorithms to escape from local optima  by identifying a new slope of increasing fitness or to completely restart the optimisation process if it cannot escape~\cite{CorusOlivetoYazdani2017,CorusOlivetoYazdani2018,HorobaJansenZarges09,JansenZarges2011c}. 
The idea is that the more of the search space that is explored, the better the ideal behaviour of the IPH is approximated through the discovery of  better and better local optima. 

Our analysis reveals that such a strategy does not produce the desired effect.
Since the mutation rate decreases with the distance to the best found local optimum the algorithm may encounter difficulties at identifying new promising optima. 
In particular, if the algorithm identifies some slope that leads away from the previous local optimum, then the mutation rate will increase as the new optimum is approached.
Firstly, this makes the new optimum hard to identify. Secondly, the high mutation rates in its proximity lead to high wastage of fitness function evaluations defeating our main motivation of reducing
such wastage compared to static mutation potentials. We rigorously prove this effect for the well-studied \textsc{TwoMax} bimodal benchmark function where the expected runtime does not improve compared to the runtime of static hypermutations to optimise each slope of the function (\emph{i.e.}, $\Theta(n^2 \log{n})$). On the other hand we use the \textsc{Cliff} benchmark function to show that the IPH can escape local optima when coupled with ageing. Static hypermutations cannot optimise the function efficiently because, once the local optimum is escaped the high mutation rates force the algorithm to jump back to the local optima.

To this end we define a {\it Symmetric} IPH operator that decreases the potential with respect to the distance to the best local optimum and uses the same rate of decrease in all other directions.
We prove the effectiveness of our strategy, and subsequent speed-ups over static hypermutations, for the \textsc{TwoMax} benchmark function while showing that the local optima of 
 \textsc{Cliff} can still be escaped by the {\it Symmetric} IPH operator.




\section{Preliminaries}
The original IPH potential proposed for Opt-IA was  $M=\lceil(1-f_{OPT}/v)\rceil cn$ for minimisation problems, where $f_{OPT}$ is the best known fitness and $v$ is the fitness of the individual. 
In an analysis for \textsc {OneMax}, such mutation potential was shown not to decrease below $cn/2$, with $cn$ being the highest possible mutation potential~\cite{JansenZarges2011}.
In this section we introduce three different IPH operators  that will be analysed in this paper. Two have been already considered in the literature 
while the third one is newly proposed by us based on the performance of the first two.




\subsection{Hamming Distance Based Linear Decrease}
 Zarges analysed an IPH operator where the probability of flipping each bit increased linearly with the Hamming distance to the optimum (or the best available estimate of the optimum)~\cite{Zarges2008}.
Precisely, this mutation operator flips each bit with probability $\min\{H(x,best)\}/n$ where $n$ is the size of problem and $\min\{H(x,best)\}$ is the minimum Hamming distance of the current point to a best individual. Here, we consider the mutation potential version of such operator. As the expected number of bit-flips is $\min\{H(x,best)\}$ during each execution of the mutation operator, we choose this value for the mutation potential:
\begin{align}
M_{\text{linHD}}(x)=\min\{H(x,best)\}.
\end{align}

\subsection{Fitness Difference Based Exponential Decrease}
In Clonalg's  
IPH 
operator 
the mutation rate decreases as an exponential function of the fitness of the current solution~\cite{DecastroVonzuben2002}. 
Precisely, each bit 
flips with probability $e^{-\rho \cdot v}$ where $v$ is the normalised fitness value and $\rho$ is a decay parameter that regulates the speed at which the mutation rate decreases.
Since we consider only maximisation problems, we use $v=\frac{f(x)}{f(best)}$ as suggested by \cite{Zarges2008} where $best$ is the best known fitness value. Using this mutation operator as a mutation potential gives $M=n \cdot e^{-\rho \cdot v}$. According to both practical and theoretical results in
\cite{Zarges2008}, a reasonable value for $\rho$ is $\ln n$. We call this mutation potential~\expoF and define it as
\begin{align} 
M_{\text{expoF(x)}}(x)=n^{1-\frac{f(x)}{f(best)}}.
\end{align}
\subsection{Hamming Distance Based Exponential Decrease}
Since it is well understood that using differences in fitness values makes randomised search heuristics unstable towards the scaling of fitness functions~\cite{OlivetoWitt2014,OlivetoWitt2015,Whitley1989},  
we also consider a mutation potential which is similar to \expoF with the exception that it uses the normalised Hamming distance to the best estimate rather than the normalised fitness. We call this mutation potential \expoHD and define it as \expoHD$=n \cdot e^{-\rho \frac{n-H(x,best)}{n}}$ where $n$ is the maximum Hamming distance from any search point to the optimum (which is always at most $n$), $H(x,best)$ is the Hamming distance to the best known individual and $\rho$ is the decay of the mutation potential. For the choice of $\rho=\ln n$, we get, 
\begin{align}
M_{\text{expoHD}}(x)=n^{(1-\frac{n-H(x,opt)}{n})}=n^{\frac{H(x,best)}{n}}.
\end{align}

\section{Ideal Behaviour Analysis}\label{sec:OptKnown} 

\begin{table*}[t!]
\caption{Comparison of the results obtained by different hypermutation schemes for \textsc{OneMax} and \textsc{LeadingOnes}. The same result as $(1+1)~IA^{hype}$ using static hypermutation ($M=n$) has been shown for the original mutation potential of Opt-IA \cite{JansenZarges2011}.}

\begin{center}
\begin{tabular}{ |l l l| }
 \hline
 Algorithms & \textsc{OneMax} &  \textsc{LeadingOnes} 
\T\B \\ \hline
$(1+1)~IA^{Hyp}$ using \linHD & $\Theta(n^2)$ & $\Theta(n^3)$
\T\B\\ 
  $(1+1)~IA^{Hyp}$ using \expoF & $O(n^{(3/2)+\epsilon}\log n)$, $\Omega(n^{3/2-2\epsilon})$ & $O(n^3/\log n)$, $\Omega(n^{5/2+\epsilon})$
 \T\B \\
$(1+1)~IA^{Hyp}$using \expoHD & $O(n^{(3/2)+\epsilon}\log n)$, $\Omega(n^{3/2-2\epsilon})$ & $O(n^{\frac{5/2+\epsilon}{\ln n}})$, $\Omega(n^{9/4-\epsilon})$
 \T\B \\
$(1+1)~IA^{hype}$ with $M=n$\cite{CorusOlivetoYazdani2017}  & $\Theta(n^2\log n)$ \cite{CorusOlivetoYazdani2017} & $\Theta(n^3)$\cite{CorusOlivetoYazdani2017}
\T\B\\
\hline
 \end{tabular}
\end{center}
\label{table:afl}
\end{table*}

 \begin{figure}[t!]
 \centering
  \includegraphics[width=0.3\textwidth]{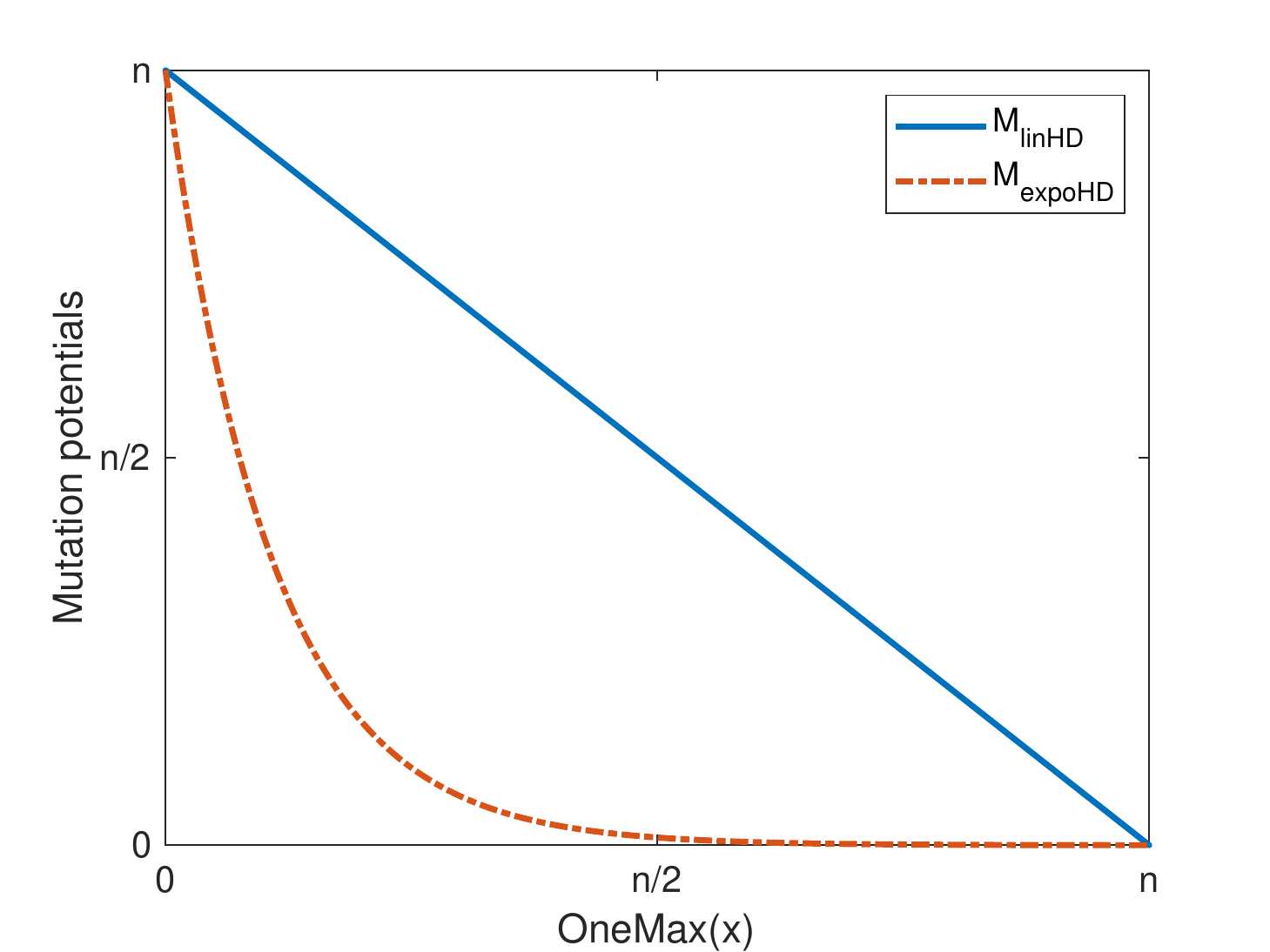}
 \caption{\textsc{Mutation potentials with respect to the distance to the optimum of \textsc{OneMax}. Since fitness  values and Hamming distance are the same for \textsc{OneMax}, we only plot one exponentially decaying curve.}}
 \label{fig:onemax}
 \end{figure}
 
   \begin{figure}[t!]
 \centering
  \includegraphics[width=0.3\textwidth]{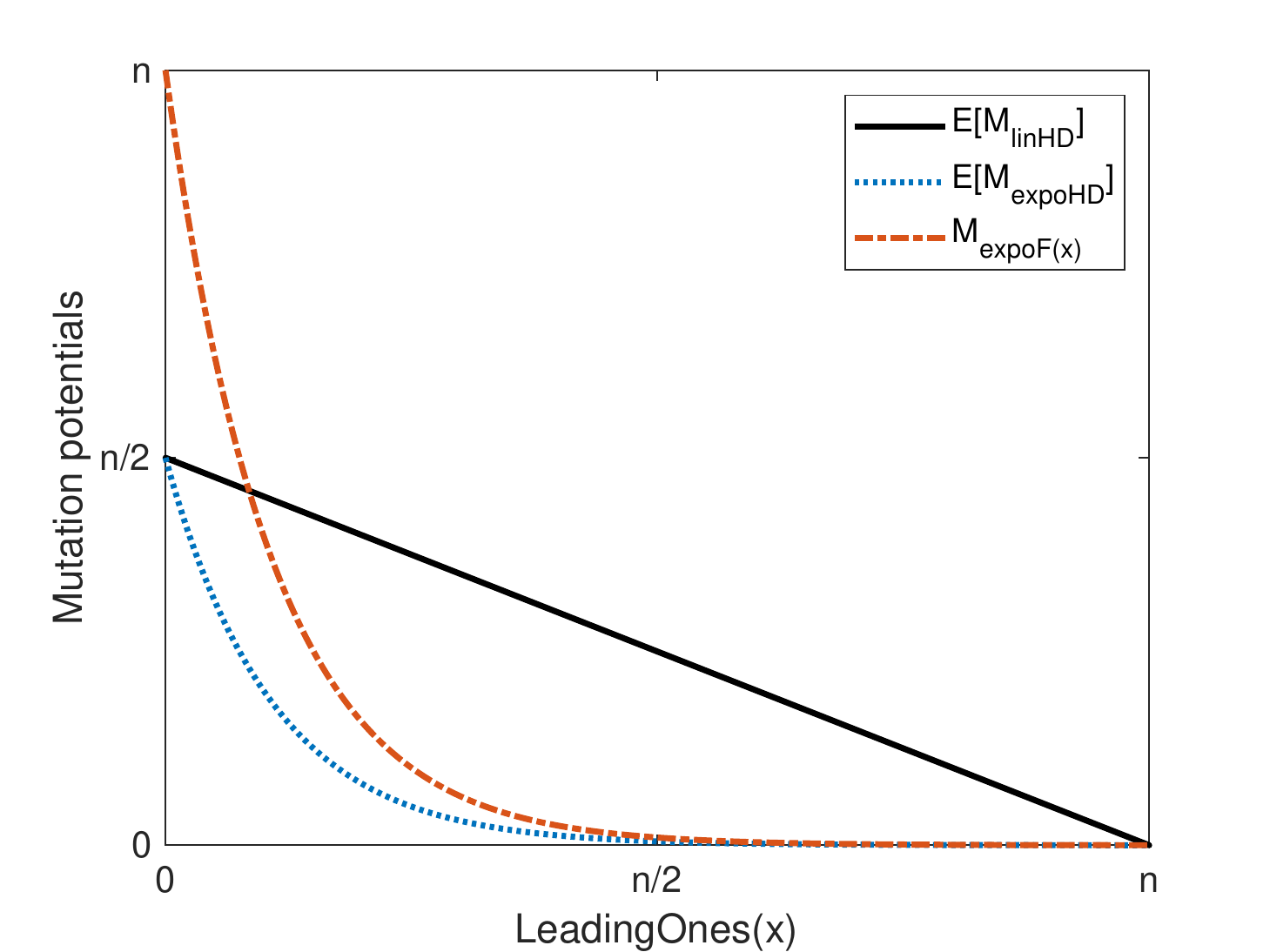}
 \caption{\textsc{Mutation potentials for \textsc{LeadingOnes} }. Expected values are provided for IPHs that use Hamming distance.}
 \label{fig:leadingones}
 \end{figure}

In this section we evaluate the performance of the different inversely proportional mutation potentials, 
assuming the optimum is known (i.e., $best=opt$). Under this assumption, the operators exhibit their ideal behaviour (i.e., the mutation potential decreases with the desired rate as the optimum is approached).
Our aim is to evaluate what speed-ups can be achieved in ideal conditions compared to the well-studied static mutation potentials. To achieve such comparisons we perform runtime analyses of the (1+1)~IA$^{hyp}$ using the IPHs on the \textsc{OneMax} and \textsc{LeadingOnes} unimodal benchmark functions for which the performance of the same algorithm using static mutation potentials is known~\cite{CorusOlivetoYazdani2017}. The simple to define \textsc{OneMax} function counts the number of 1-bits in the bit string and is normally used to show the hill climbing ability of the algorithm. On the other hand, \textsc{LeadingOnes}, is a more complicated unimodal problem which counts the consecutive number of 1-bits at the beginning of the bit string before the first 0-bit. 

The pseudo-code of the algorithm, \oneoneiahype, is given in Algorithm~\ref{alg:simple}. It simply uses one candidate solution in each iteration to which inversely proportional hypermutations are applied.
The results proven in this section and comparisons with static hypermutations are summarised in Table~1.
%
%
As previously mentioned, {\it Constructive mutation} in Algorithm \ref{alg:simple} is a mutation step where the evaluated bit string is at least as fit as its parent. 

Figures \ref{fig:onemax} and \ref{fig:leadingones} show how the studied mutation potentials decrease during the run of the algorithm when optimising \textsc{OneMax} and \textsc{LeadingOnes}, respectively. 


\begin{algorithm}[t] 
    \caption{\oneoneIA}
    \begin{algorithmic}[1] 
            \STATE{Initialise $x \in\{0,1\}^n$ uniformly at random;}
            \STATE Evaluate $f(x)$;      
            \WHILE{termination condition is not satisfied} 
            \STATE{ $M = IPM(x) $};
            \STATE{Create $y$ by flipping at most $M$ distinct bits of $x$ selected uniformly at random one after another until a \textit{constructive mutation} happens;};
			\IF {$f(y) \geq f(x) $}
			\STATE{$x:=y$}.
			\ENDIF
            \ENDWHILE
    \end{algorithmic} \label{alg:simple}
\end{algorithm}

\subsection{\textsc{OneMax}}
The following three theorems derive the expected runtimes of the three variants of IPM for the function $\textsc{OneMax}(x):=\sum_{i=1}^{n}x_i$. By decreasing the potential linearly with the decrease of the Hamming distance, a logarithmic factor may be shaved off from the expected runtime of the \oneoneIA compared to the expected runtime if static hypermutation potentials were used.

The following lemma, called Ballot theorem, is used throughout this paper to derive lower bounds on the runtime of the algorithms for \textsc{OneMax}$(x):=\sum_{i=1}^n x_i$.  It was first used to analyse hypermutations with mutation potential for the same function by Jansen and Zarges \cite{JansenZarges2011}. This theorem is essentially used to show the probability of picking more 1-bits than 0-bits during one hypermuation operation given that there are more 1-bits in the initial bit-string (or vice-versa). 

\begin{lemma}[Ballot Theorem~\cite{feller1968}] \label{thm:ballot} 
In a ballot, suppose that candidate P receives $p$ votes and candidate Q receives $q$ votes, such that $p>q$. The probability that throughout the counting P is always ahead of Q is $(p-q)/(p+q)$. 
\end{lemma}

\begin{theorem} \label{th:linHD-OM}
The \oneoneIA using \linHD optimises \textsc{OneMax} in $\Theta(n^2)$ expected fitness function evaluations.
\end{theorem}

\begin{proof}
Considering $i$ as the number of 0-bits in the candidate solution, the probability of improvement in the first step is $i/n$. Knowing that at most $n$ improvements are needed to find the optimum and in case of failure $H(x,opt)=i$ fitness function evaluations will be wasted, the total expected time to optimise \textsc{OneMax} is at most $\sum_{i=1}^n \frac{n}{i} \cdot i =O(n^2)$.

In order to prove a lower bound, we use Ballot theorem \cite{feller1968} which is stated in Lemma \ref{thm:ballot}. By Chernoff bounds the number of 0-bits in the initialised solution is at least $n/3$ w.o.p. Considering the number of 
0-bits as $i=q$ and the number of 1-bits as $n-i=p$, the probability of 
an improvement is at most $1-(p-q)/(p+q)=1-(n-2i)/n=2i/n$ by Lemma \ref{thm:ballot} where $i=H(x, opt)$. This means that we need to wait at least $n/(2i)$ iterations to see an improvement and each time the mutation operator fails to improve the fitness, $i$ fitness function evaluations will be wasted. Considering that at least $n/3$ improvements are needed, the expected time to optimise \textsc{OneMax} is larger than $\sum_{i=1}^{n/3} \frac{n}{2i} \cdot i= \Omega(n^2)$.
\end{proof}

The following theorem shows that a greater speed up may be achieved if the potential decreases exponentially rather than linearly. Its proof  deviates from the proof of Theorem 3.2 only in the amount of wasted evaluations. Note that for \textsc{OneMax} the Hamming distance of a solution to the optimum and its difference in fitness are the same. Hence, the subsequent corollary is obvious.

\begin{theorem} \label{th:expoF(x)-UPonOM}
The \oneoneIA using \expoF optimises \textsc{OneMax} in $O(n^{3/2+\epsilon} \log 
n)$ and $\Omega(n^{3/2-\epsilon})$ expected fitness function evaluations for 
any arbitrarily small constant $\epsilon>0$.
\end{theorem}

\begin{corollary}\label{cor:omexpohd}
The \oneoneIA using \expoHD optimises \textsc{OneMax} in $O(n^{3/2+\epsilon} 
\log 
n)$ and $\Omega(n^{3/2-\epsilon})$ expected fitness function evaluations for 
any arbitrarily small constant $\epsilon>0$.

\end{corollary}




\subsection{\textsc{LeadingOnes}}
In the previous section it was shown that decreasing the mutation rate linearly with the Hamming distance to the optimum gave a small improvement for \textsc{OneMax}. The following theorem shows that no improvement over static mutation potentials are achieved for \textsc{LeadingOnes}$:= \sum_{i=1}^{n}\prod_{j=1}^{i}x_i$. The main reason for the lack of asymptotic improvement is that a linear mutation potential is sustained until at least a linear number of improvements are achieved.

\begin{theorem} \label{th:linHD-LO}
The \oneoneIA using mutation potential \linHD optimises \textsc{LeadingOnes} in $\Theta(n^3)$ expected fitness function evaluations.
\end{theorem}

\begin{proof}
The probability of improvement in each step is at least $1/n$ which is the probability of flipping the leftmost 0-bit. As at most $n$ improvements are needed and each failure in improvement yields $n$ wasted fitness function evaluations (as the $H(x,opt)$ is at most $n$), the expected time to find the optimum is at most $\sum_{i=1}^n n \cdot n= O(n^3)$.

The proof for lower bound is the same as the proof of Theorem 3.6 in \cite{CorusOlivetoYazdani2017} for the expected runtime of the $(1+1)~IA^{Hyp}_{\geq}$\footnote{A simple (1+1)~EA algorithm with static hypermutation operator (i.e., $M=n$) instead of standard bit mutation. The selection mechanism accepts equally fit solutions.} with the exception that the wasted amount of fitness function evaluations in case of failure is now $H(x,opt)$ instead of $n$.

Considering $i$ as the number of leading 1-bits, we show the expected number of fitness function evaluations until an improvement happens by $E(f_i)$. Any candidate solution has $i$ leading 1-bits with a 0-bit following, and then $n-i-1$ other 0-bits and 1-bits which are distributed uniformly at random \cite{DrosteJansenWegener2002}. We take into account three possible events of $E_1$, $E_2$ and $E_3$ that can happen in the first bit flip; $E_1$ is the event of flipping a leading 1-bit which happens with probability $i/n$, $E_2$ is the event of flipping the first 0-bit which happens with probability $1/n$, and $E_3$ which is the event of flipping any other bit which happens with probability $(n-i-1)/n$. So we get $E(f_i|E_1)=E[H(x,opt)]+ E(f_i)$ , $E(f_i|E_2)=1$, and $E(f_i|E_3)=1+ E(f_i)$.
With law of total expectation and considering that the expected value of $H(x,opt)$, i.e, $E[H(x,opt)]$, is $(n-i)/2 +\epsilon n$, we get $ E(f_i) = \frac{i}{n} \left( \frac{n-i}{2}+\epsilon n+ E(f_i)\right) + \frac{1}{n} \cdot 1 +\frac{n-i-1}{n} \left( 1+ E(f_i)\right)$. Solving it for $E(f_i)$ gives us $E(f_i)=\frac{in-i^2+2n-2i+2\epsilon i}{2}$. We know that the expected number of consecutive 1-bits that follow the leftmost 0-bit is less than two \cite{DrosteJansenWegener2002} which means the probability of not skipping a level $i$ is $\Omega(1)$. The initial 
solution on the other hand will have more than $n/2$ leading ones with 
probability at most $2^{-n/2}$. Thus, we obtain a lower 
bound $(1-2^{-n/2})\sum_{i=n/2}^{n} E(f_i)=\Omega(1) \sum_{i=n/2}^{n} \frac{in-i^2+2n-2i +2\epsilon i}{2}=\Omega(n^3)$ on the expectation. 
\end{proof}

The following two theorems show that exponential fitness-based mutation 
potential gives us at least a logarithmic and at most $\sqrt 
n$ factor speed-up compared to static mutation potentials. Before proving the main results, we introduce the following  lemma that will be used in the proof of upcoming theorems.
\begin{lemma}\label{lem:expo}
 For large enough $n$ and any arbitrarily small constant $\epsilon$, 
$n^{1/n^{\epsilon}}= (1+\frac{\ln{n}}{n^{\epsilon}})(1 \pm o(1))$.
\end{lemma}
\begin{proof}
By raising $\left(1+\frac{\ln{n}}{n^{\epsilon}}\right)$ to the power 
of $\frac{n^{\epsilon}}{\ln{n}} \cdot \frac{\ln{n}}{n^{\epsilon}}$ we have
$\left(1+\frac{\ln{n}}{n^{\epsilon}}\right) ^{\frac{n^{\epsilon}}{\ln{n}} 
\cdot \frac{\ln{n}}{n^{\epsilon}}} = \left(1 \pm 
o(1)\right)e^{\frac{\ln{n}}{n^{\epsilon}}}=(1 \pm 
o(1))n^{1/n^\epsilon}. $\end{proof}

\begin{theorem} 
The \oneoneIA using \expoF optimises \textsc{LeadingOnes} in $O(n^3/\log n)$ and 
$\Omega(n^{5/2+\epsilon})$  expected fitness function evaluations for any 
arbitrarily small constant $\epsilon>0$. 
\end{theorem}

\begin{proof}
The expected number of leading 1-bits is less than two in the initialised bit 
string. The probability of improvement in the first step is $1/n$. In case of 
failing in improving at the first step, at most $n^{\frac{n-i}{n}}$ fitness 
function evaluations would get wasted with $i$ showing the number of leading 
1-bits. Therefore, the total expected time to find the optimum is  $E(T) \leq 
\sum_{i=1}^{n} n \cdot n^{(n-i)/n}=O( n^3/\log n)$ considering that $ 
\sum_{i=1}^{n} n^{(i)/n}\leq \sum_{i=0}^{\infty} (n^{1/n})^i= 1/(1-n^i) $ which is $n^2/(\log n)$ by Lemma~\ref{lem:expo}.

The proof for lower bound is similar to the proof of Theorem \ref{th:linHD-LO}, except in the calculation of $E(f_i)$ when we want to consider the amount of wasted fitness function evaluations in case of $E_1$ happening. Here we have $ E(f_i) = \frac{i}{n} \left( n^{(n-i)/n}+ E(f_i)\right) + \frac{1}{n} \cdot 1 +\frac{n-i-1}{n} \left( 1+ E(f_i)\right)$. Solving it for $E(f_i)$ gives us $E(f_i)=in^{i/n}+n-i$. 
Hence the expected time to optimise \textsc{LeadingOnes} is 
$(1-2^{-n/2})\sum_{i=n/2}^{n} E(f_i)=\Omega(1) \sum_{i=1}^{n/2-\epsilon n/2} 
in^{(n-i)/n}+n-i=\\ \Omega(1)\left( \sum_{i=1}^{n/2-\epsilon n/2} n-i+ 
\sum_{i=1}^{n/2-\epsilon n/2} in^{(n-i)/n}\right)$. Evaluating the second 
sum in the interval $i\in [n/2-\epsilon n/2, n/2-\epsilon]$, we get $\epsilon 
n/2 \cdot (n/2 - \epsilon n)  \cdot 
n^{1/2-\epsilon}=\Omega(n^{5/2-\epsilon})$.
\end{proof}

An advantage of Hamming distance-based exponential decays of the mutation potential
compared to fitness-based ones are provided by the following theorem for \textsc{LeadingOnes}.
The reason can be appreciated from Fig. \ref{fig:leadingones}. While the initial fitness is very low, hence the potential is very high, the actual number of bits that have to be flipped to reach the optimum is much smaller. \expoHD exploits this property, thus wastes less fitness evaluations than \expoF.
. 
\begin{theorem}\label{thm:expohdLO}
The \oneoneIA using \expoHD optimises \textsc{LeadingOnes} in 
$O(\frac{n^{5/2+\epsilon}}{\ln n})$ and $\Omega(n^{9/4-\epsilon})$ expected 
fitness function evaluations for any arbitrarily small constant $\epsilon>0$. 
\end{theorem}

Given that \expoHD provides larger speed-ups  compared to the other IPH operators and is stable to the scaling of fitness functions,
we will use it in the remainder of the paper.

\section{Realistic Inefficient Behaviour}In this section we consider the usage of  \expoHD in realistic applications where the optimum is unknown.
To this end the best found solution will be used by the operator rather than the unknown optimum.
We combine \expoHD with hybrid ageing, as in the Opt-IA AIS and call it $(1+1)$~Opt-IA~\cite{CutelloTEVC}. It's pseudo-code is provided in Algorithm \ref{alg:alg1+ageing}.
Ageing has been shown to enable algorithms to escape from local optima either by identifying a gradient leading away from it or by restarting the whole optimisation process.

Our aim is that the more local optima are identified by the algorithm, the more \expoHD approximates its ideal behaviour.
However, we will show that this is not the case using the  well-studied
bimodal benchmark functions \textsc{TwoMax} (\ref{func:twomax})~\cite{FriedrichOSWECJ09, SudholtBookChapter2019, OlivetoSudholtZarges2018, CovantesOsunaSudholt2017} as an example:
\begin{equation}\label{func:twomax}
\textsc{TwoMax}:=\max \left\{\sum_{i=1}^n x_i, \; n-1 \sum_{i=1}^n x_i \right\}
\end{equation} 
The function is usually used to evaluate the global exploration capabilities of evolutionary algorithms i.e., whether the populations identify both optima of the function.
Our analysis shows that once the $(1+1)$~Opt-IA escapes from one local optima, the mutation rate will increase as the algorithm climbs up the other branch.
As a result the algorithm struggles to identify the other optimum and wastes more and more fitness function evaluations as it approaches it. 
Thus, defeating the whole purpose behind IPH.

On the bright side, we will show that \expoHD combined with ageing can escape from local optima. We will use the well known \textsc{Cliff} function for the purpose
where static hypermutations are inefficient because due to their high mutation rates, even if they escape, they jump back to the local optima with high probability.  
 \textsc{Cliff} is defined as follows:
 \begin{equation} \label{func:cliff}
\textsc{Cliff}_{k}(x)=\begin{cases}
\sum_{i=1}^n x_i & \text{if}\; \sum_{i=1}^n x_i \leq n-k, \\
\sum_{i=1}^n x_i -k+ 1/2 & \text{otherwise.}
\end{cases}
\end{equation}
Both bimodal functions are illustrated in Fig \ref{fig:benchmarks}. 


 Among the different varaints of ageing, {\it hybrid} ageing has been shown to be very efficient at escaping local optima~\cite{OlivetoSudholt2014,CorusOlivetoYazdani2017,CorusOlivetoYazdani2018}. Using this operator, the individual is assigned with an initial $age=0$. During each iteration of the algorithm the age  increases by 1 and is passed to the offspring if the offspring does not improve over its parent's fitness. If the offspring is better than the parent, then its age is set to 0. At the end of each iteration any individual with age larger than a threshold ($\tau$) is removed with probability $1/2$ and in case there is no other individual left in the population, a new individual is initialised uniformly at random. 

\begin{algorithm} 
    \caption{$(1+1)$~Opt-IA with \expoHD}
    \begin{algorithmic}[1] 
             \STATE{Initialise $x \in\{0,1\}^n$ uniformly at random and add to $P$;}
            \STATE{Set $best=x$};      
            \STATE Evaluate $f(x)$;
            \STATE{Set $x.age=0$};      
            \WHILE{termination condition is not satisfied} 
             \STATE{$x.age=x.age+1$};
             \STATE{ $M =$\expoHD};
            \STATE{Create $y$ by flipping at most $M$ distinct bits of $x$ 
selected uniformly at random one after anotherther until a \textit{constructive 
mutation} happens;}
            \STATE{If $f(y) > f(x)$, then $y.age=0$. Else, $y.age=x.age$};
            \STATE{If $f(y)\geq f(best)$, then set $best=y$}
            \STATE{Add $y$ to $P$};
			\STATE{For $x$ and $y$ in $P$, if $age \geq \tau$ then remove the individual with probability $1/2$}; 
			\STATE {Select the best individual in $P$ and remove the other};
			\STATE{If $|P|=0$, create $x \in\{0,1\}^n$ uniformly at random and add to $P$;}
            \ENDWHILE
    \end{algorithmic} \label{alg:alg1+ageing}
\end{algorithm}

 \begin{figure}[t!]
 \centering
\includegraphics[width=0.4\textwidth]{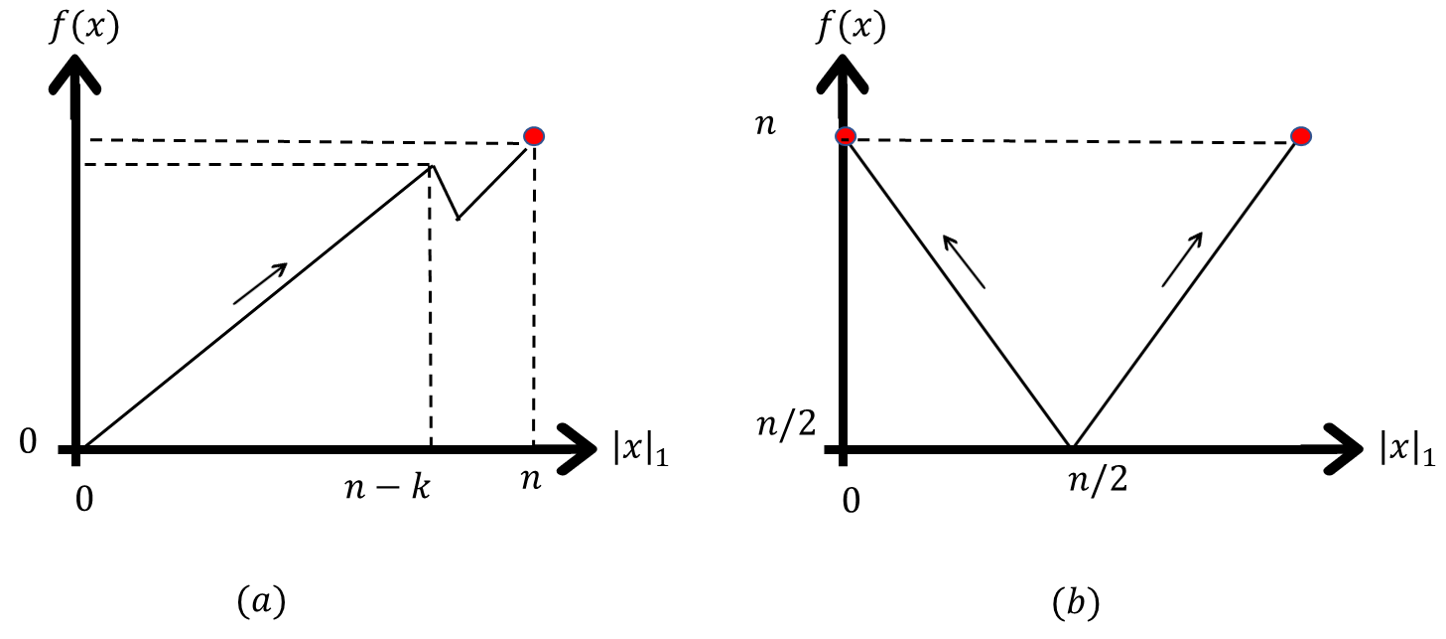}
\caption{(a) \textsc{Cliff} and (b) \textsc{TwoMax} test functions}
\label{fig:benchmarks}
 \end{figure}
%

The following theorem shows that after the algorithm escapes from the local optimum the mutation rate increases as the algorithm climbs up the opposite branch.
This behaviour causes a large waste of fitness evaluations defying the objectives of IPH.
\begin{theorem} \label{thm:twomaxexpo}
The expected runtime of Algorithm \ref{alg:alg1+ageing} with \expoHD to optimise 
\textsc{TwoMax} is $O(n^{2}\log n)$ with $\tau=\Omega(n^{1+\epsilon})$ for some 
constant $\epsilon>0$.
\end{theorem}
\begin{proof}
Let $x_t$ be the current solution at the beginning of the iteration $t$. Immediately after the initialisation, the 'best' seen is the current individual $x_1$ itself and the mutation potential is $M=M_{expoHD}=n^{HD(x_1,x_1)/n}=n^0=1$. Note that  $x_t \neq x_{t+1}$ if and only if either $f(x_{t+1})\geq f(x_t)$ or $x_{t+1}$ is reinitialised after $x_t$ is removed from the population due to ageing. Thus, the mutation operator flips a single bit at every iteration until ageing is triggered for the first time.
 The improvement probability will be at least $ 1/n $ until either $1^n$ or $0^n$ is sampled. 
Given that the ageing threshold $\tau$  is at least $n^{1+\epsilon}$ for some constant $\epsilon >0 $, 
the probability that the current solution will not improve $\tau$ times consecutively is at most 
$(1-\frac{1}{n})^{n^{1+\epsilon}}$ $e^{-\Omega(n^\epsilon)}$. 
Hence, with overwhelmingly high probability, the first optimum will be found before ageing is triggered. Given that the ageing operator is not triggered, the expected time to find the first optimum is at most $O(n \log{n})$ as for the standard \oneonerlsone \footnote{\oneonerlsone or random local search flips exactly one bit at each iteration as the variation operator.} . After finding the first optimum, no single-bit flip can yield an equally fit solution. The individual then reaches age $\tau $ and the ageing reinitialises the current solution. Thus, with 
$\left(1-2^{n}\right) \cdot \left(1-e^{-\Omega(n^\epsilon)} \right) $ probability a new solution will be initialised in $O(n\log{n}) + \tau + n$ and the current \emph{best} will be the first discovered optimum. Let the first and second branch denote the subsets of the solution space which consist of solutions with less than and more than Hamming distance $n/2$ to the first discovered optimum respectively. 

We first consider the case if the reinitialised solution is on the first branch. Given that the current solution has fitness $i$, the distance to the best seen is $n-i$ and the mutation potential is $M=n^{\frac{n-i}{n}}$. Since, the first constructive mutation ensures that the probability of improvement is always at least $1/n$, with overwhelmingly high probability the ageing will not be triggered until either optimum is discovered. Thus, given that the current solution is always on the first branch, the proof of Theorem~\ref{cor:omexpohd} carries over and the expected time to find the first discovered optimum once again is at most $O(n^{3/2}\log{n})$ in expectation. When the first discovered optimum is sampled again, with overwhelming probability, the current solution is reinitialised with the first discovered optimum as the best seen solution in at most $\tau + n$ iterations.

Now, we consider the case where the current solution is on the second branch. A lower 
bound on the probability that the current solution will reach the optimum of the second branch before 
sampling an improving solution in the first branch will conclude the proof since the expected time to do so is $O(n^{3/2}\log{n})$ given that the current solution does not switch branches before. 

We will start by bounding the mutation potential for a solution in the second branch with fitness value $n-k$. Since the  \expoHD is always smaller than \linHD, we can assume that no more than $n-k$ bits will be flipped.

Without losing generality, let the second branch be the branch with more 0-bits. We will now consider the final solution sampled by the hypermution operator since it has the highest probability of finding a solution with at least $n-k$ 1-bits (\emph{i.e.}, switching to the first branch). The current solution has $n-k$ 0-bits and $k$ 1-bits.
If more than $k/2$ 1-bits are flipped, then the number of $1$s in the final solution after $n-k$ mutation steps is less than $n-k$ since the number  of $0$s flipped to $1$s is less than $n-k-k/2$ and the number of remaining 1-bits is less than $k/2$. The event that at most $k/2$ 1-bits are flipped is equivalent to the event that in a uniformly random permutation of the $n$ bit positions at least $k/2$ $1$-bits are ranked in the last $k$ positions 
of the random permutation. We will now bound the probability that exactly $k/2$ $1$-bits are in the last $k$ position since having more has a smaller probability. Each particular outcome of the last $k$ positions has the equal probability of $\prod_{i=0}^{k-1}(n-i)^{-1}$. There are $\binom{k}{k/2}$ different equally like ways to chose $k/2$ $1$-bits and $k!$ different permutations of the last $k$ positions thus the probability of having exactly $k/2$ 1-bits in the last $k$ positions is: 
$
\prod\limits_{i=0}^{k-1}\frac{1}{n-i}\cdot \binom{k}{k/2}\cdot k!=\prod\limits_{i=0}^{k-1}\frac{1}{n-i}\cdot \left(\frac{k!}{(k/2)!}\right)^{2}< \left( \frac{k}{n-k}\right)^k
$.
For any $k \in [4,\frac{n}{2}-\Omega(\sqrt{n})]$ this probability is in the 
order of $\Omega(1/n^4)$. Using union bound over the probabilities of having 
more than $k/2$ $1$-bits and the probabilities of improving before the final 
step, we obtain an $\frac{k}{2}\cdot (n-k)\cdot \Omega(1/n^4)= \Omega(1/n^2)$. 
Since the new solutions are uniformly sampled, the number of bits in the 
solutions are initially distributed binomially with parameters $n$ and $1/2$ 
which has a variance in the order of $\Theta(\sqrt{n})$ and implies that with 
constant probability $k <\frac{n}{2}-\Omega(\sqrt{n})$ in the initial solution. 
Given that the expected time in terms of generations is in the order of $O(n 
\log{n})$, the total probability of switching branches while 
$k\in[k,\frac{n}{2}-\Omega(\sqrt{n})]$ is at most 
$\left(1+1/n^2\right)^{O(n\log{n})}=1-\Omega(1)$. Finally, we will consider the 
cases of $k\in\{1,2,3\}$ separately. When $k=1$, the probability of flipping 
less than $k/2$ 1-bits is equivalent to the probability of not flipping the 
single 1-bit which happens with probability $1/n$. For $k=2$, similarly we have 
to flip at most one 1-bit with probability $O(1/n)$. For $k=3$, it is necessary 
that at least two 1-bits are not flipped which happens with probability at most 
$O(1/n^2)$ probability. Thus, the probability of switching branches when $k<4$ 
is at most $O(1/n)$, which lower bounds the total probability of switching 
branches in the order of $1-\Omega(1)$ given that the initial solution has at 
least $k> \frac{n}{2}-\Omega(\sqrt{n})$ 1-bits (which also occurs with at least 
$\Omega(1)$ probability). Given that no switch occurs, the expected time to find 
the second optimum is $\sum_{i=1}^{n/2+\epsilon n} \frac{n}{i} \cdot 
n^{\frac{n-i}{n}}=n^2\sum_{i=1}^{n/2-\epsilon n} \frac{1}{i} \cdot 
n^{\frac{-i}{n}}<n^2\sum_{i=\epsilon n}^{2 \epsilon n} \frac{1}{2n} \cdot 
n^{\frac{-i}{n}}<$ at most $O(n^{2}\log{n})$ as  $n^{-i/n} \leq (1-\ln{n}/n)$ 
for all $i$.
\end{proof}

Now we show that, differently from static hypermutations \expoHD 
combined with ageing can escape from the local optima of \textsc{Cliff}, 
hence optimise the function efficiently.

\begin{theorem}\label{thm:cliffexpo}
Algorithm \ref{alg:alg1+ageing} with \expoHD and $\tau=\Omega(n^{1+\epsilon})$ 
for 
an arbitrarily small constant $\epsilon$  optimises \textsc{Cliff}$_k$ with 
$k<n(\frac{1}{4}-\epsilon)$ and $k=\Theta(n)$ in $O(n^{3/2}\log{n}+ \tau 
n^{1/2} + \frac{n^{7/2}}{k^2})$ 
fitness function evaluations. 
\end{theorem}
\begin{proof}
The analysis will follow a similar idea to the proof of 
Theorem~\ref{thm:twomaxexpo}. After initialisation the initial mutation 
potential is $M=1$ since the current solution is the best seen solution. With 
single bit-flips it takes in expectation at most $O(n)$ to find on 
the local optima of the cliff (a search point with $n-k$ 1-bits) as the 
improvement probability is always at least $(n-k)/n = \Omega(1)$. Since the 
local optima cannot be improved with single bit flips, in $\tau$ generations 
after it was first discovered the ageing will be triggered and in the following 
$n$ steps the current solution will be removed from the population due to 
ageing with probability at least $1-2^{-n}$. Hamming distance of the  
reinitialised soluton will be distributed binomially with parameters $n$ and 
$1/2$ and with overwhelmingly high probability will be smaller than $n/2 + 
n^{2/3}$, yielding an initial mutation potential of $M=\Omega(n^{1/2})$. We 
pessimistically assume that the mutation potential will not decrease until the 
local optima is found again, which implies that the expected time will be at 
most $O(n^{3/2}\log{n}+\tau n^{1/2})$ since each iteration will waste an extra 
$\Omega(n^{1/2})$ fitness function evaluations. After finding a local optima 
again, the mutation potential will be $M=1$ since it will replace the 
previously observed local optima as the best seen. The process of 
reinitialisation and reaching the local optima will repeat itself until the 
following event happens.

If the local optima produces an offspring with $n-k+1$ bits  with probability 
$k/n$ and if this solution survives the ageing operator with 
probability $(1-p_{die}) $, then  the reinitialised solution will be rejected 
since its fitness value is less than $n-k$ with overwhelming probability. 
The Hamming distance of this new solution to the best seen will be exactly one 
since it is created via a single bit-flip, thus itsmutation potential 
will be $M=1$.  Moreover, if the surviving offspring improves again with 
proability $(k/n)$ in the next iteration, it will reset its age to zero and 
will 
have Hamming Distance at least two to any local optima.  In expected 
$O(n/\log{n})$ iterations (not function evaluations), this solution will reach 
the global optimum unless a solution with less than $n-k$ 1-bits is sampled 
before. Initially this will be impossible since $M=1$ for at least $\omega(1)$ 
more steps and later $M<3$ as long  as the distance to the last seen local 
optima is at most $n/\ln{n}$ since $n^{\frac{n/\ln{n}}{n}}=e$. 
Note that the number of $1$s does not always reflect the actual Hamming 
distance since more than one bit can be flipped in an accepted offspring. 
We will pessimitically assume that all improvements has increased the Hamming 
distance by three until the total Hamming distance reaches $n/\ln{n}$, which 
implies that there has been $n/(3\ln{n})$ accepted solutions.  Ballot theorem 
implies that sampling a solution at least as good as the parent (which are the 
only solutions that are accepted) has probability at most $2 i /n$ where $i$ is 
the number of 0-bits in the solution. Since the probability of improving in the 
first step is at least $i/n$, we can conclude that the conditional probability 
that an accepted offspring is an improvement is at least $1/2$. Thus, when the 
Hamming distance to the local optima reaches $n/\ln{n}$, in expectation the 
current solution will have at least $n/(6\ln{n})$ extra 1-bits compared to 
the local optima and at least $n^3/5$ extra bits with overwhelmingly high 
probability. The Hamming distance to the local optima can be at most $2k$ 
since both the local optima and the current solution have less than $k$ 
0-bits. Since $k<n/4$ the the mutation potential is at most $\sqrt{n}$, thus, 
no hypermutation can yield a solution with less than $n-k+2$ 1-bits. Therefore, 
once a solution with $n-k+2$ bits is added to the population the algorithm 
finds 
the optimum with overwhelming probability in $O(n\log{n})$ iterations and in at 
most $O(n^{3/2}\log{n})$ fitness functions evaluations since the the mutation 
potential is at most $\sqrt{n}$. The probability of obtaining a solution with 
$n-k+2$ 1-bits at the end of each cycle of reinitialisation and removal of the 
local optima due to ageing  is $(1-p_{die}) (k/n)^2$. Since each such cycle 
takes $O(n^{3/2})$ fitness function evaluations our claim follows.
\end{proof}

\section{An Efficient Opt-IA with IPH}
In the previous section we observed in the analyses of both  \textsc{Cliff} and \textsc{TwoMax} that towards the end of the optimisation process the mutation potential increases as the current solution approaches an undiscovered, potentially promising optimum. This behaviour is against the design intentions of the inversely proportional mutation potential since in the final part of the optimisation process it gets harder to find improvements and high mutation potentials lead to many wasted fitness function evaluations.
The underlying reason of this behaviour in both \textsc{Cliff} and \textsc{TwoMax} landscapes was the necessity to follow a gradient away from the local optimum for finding the global one. Considering that this necessity would be ubiquitous in optimisation problems we propose a new method to control mutation potentials in this section. Algorithm~\ref{alg:hyp} uses a mutation potential inversely proportional to the current solution's Hamming distance to its origin, where the origin is defined as the ancestor of the current bitstring after the last removal of a solution due to ageing. 
We call the newly proposed mutation operator {\it Symmetric \expoHD}.
This mutation potential reliably decreases (at the same rate it would use if it was approaching the currently best seen local optimum) as the current solution improves and moves away from its origin up until it starts doing local search and finds a local optimum. Every time a local optimum is found, ageing is triggered after approximately $\tau$ steps and then both surviving and reinitialised individuals reset their origin to their own bitstring. 
 \begin{figure}[t!]
 \centering
  \includegraphics[width=0.15\textwidth]{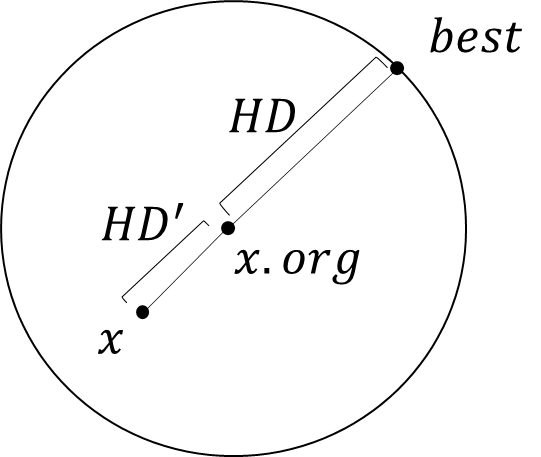}
 \caption{The geometric representation of the mutation potential of 
Algorithm~\ref{alg:hyp}. The mutation potential is determined according to the 
ratio of the Hamming distance between the current solution and its origin and 
the Hamming distance between its origin and the best seen solution.}
 \label{fig:plot}
 \end{figure}



\begin{algorithm} 
    \caption{(1+1) Opt-IA with Symmetric \expoHD}
    \label{alg:hyp}
    \begin{algorithmic}[1] 
            \STATE{Set $x \sim Unif(\{0,1\}^n)$ }
            \STATE {Set $x.origin:=x$, $x.age:=0$; $Best:= x$}
            \STATE {Evaluate $f(x)$;   } 
            \WHILE{termination condition not satisfied} 
            \STATE{$x.age:=x.age+1$};
            \STATE{Create $y$ by flipping at most $M:= \left\lceil n \cdot 
n^{-\frac{HD\left(x,x.org\right)}{\max\{HD\left(Best,x.org\right), 1\}}} 
\right\rceil$ distinct bits of $x$ selected uniformly at random one after 
another until a \textit{constructive mutation} happens;} 
			\STATE{$y.origin:=x.origin$};
			\IF{$f(y)> f(x)$}
			\STATE{$y.age:=0$};
			\IF{$f(y)\geq Best$}
			\STATE {$Best:=y$;}
			\ENDIF
			\ELSE
			\STATE{$y.age:=x.age$};
			\ENDIF
			\FOR{$w in \{x,y\}$}
			\IF { $w.age > \tau \wedge p_{die}=1/2>R \sim 
Unif(0,1)$}
			\STATE{Set $w \sim Unif(\{0,1\}^n)$ , $w.age=0$, 	}
			\STATE{Set $x.origin=x$, $y.origin=y$}
			\ENDIF
			\ENDFOR
			\STATE{Set $x=\arg\max\limits_{z\in\{x,y\}}f(z)$; }
            \ENDWHILE
    \end{algorithmic}
\end{algorithm}

The following theorem shows that once one local optimum has been identified
the mutation potential of Symmetric \expoHD decreases as both optima are approached as desired and the wished for speed-up in the runtime is achieved.
\begin{theorem}
The Algorithm~\ref{alg:hyp} with $\tau=\Omega(n^{1+\epsilon})$ for any 
arbitrarily small constant $\epsilon>0$ has expected runtime of 
$O(n^{3/2}\log{n})$ on \textsc{TwoMax}.
\end{theorem}
  \begin{proof}
The expected time until the first branch is optimised is $O(n\log{n})$ 
since the best seen search point is the current best individual and 
consequently the mutation potential is $M=1$. Since the improvement 
probability is at least $1/n$ and $\tau=\Omega(n^{1+\epsilon})$, the ageing 
operator does not trigger before finding one of the optima with overwhelming 
probability. Once one of the optima is found, the ageing reinitialises the 
individual while the first discovered optima stays as the current best seen 
search point. For the randomly reinitialised solution, the Hamming distance to 
the best seen binomially distributed with parameters $n$ and $1/2$. Using 
a Chernoff bound we can bound the distance to the previously seen 
optima by at most $n/2 + n^{2/3}$ with overwhelmingly high probability. This 
Hamming distance implies an initial mutation potential of $M< 
n^{\frac{\frac{n}{2}+n^{2/3}}{n}}=n^{\frac{1}{2}+\frac{1}{n^{1/3}}}$ which 
decreases as the individual increases its distance to the origin and can never 
go above its initial value where the distance is zero. 
Pessimistically assuming that the mutation potential will be 
$n^{\frac{1}{2}+\frac{1}{n^{1/3}}}=O(n^{1/2})$ throughout the run, we can obtain 
the above upper bound by summing over all levels and using coupon collector's 
argument.
\end{proof}

The following theorem shows that Symmetric \expoHD is also efficient for \textsc{Cliff}
\begin{theorem}
The Algorithm~\ref{alg:hyp} with $\tau=\Omega(n^{1+\epsilon})$ has expected runtime of 
$O(n^{3/2}\log{n}+ \tau n^{1/2} + \frac{n^{7/2}}{k^2})$ on \textsc{Cliff}$_k$. 
\end{theorem} 

\begin{proof}
The proof of the result is almost identical to the proof of 
Theorem~\ref{thm:cliffexpo}. The most important distinction is that once a 
solution with $n-k+2$ is created its mutation potential remains at $M=1$ until 
it finds the global optimum because when the ageing triggers the surviving 
solutions all reset their origin to their own bitstring, \emph{i.e.} start 
doing randomised local search. 
%
%
\end{proof}

\section{Conclusion}
We have presented an analysis of Inversely Proportional Hypermutations (IPH).
Previous theoretical studies have shown disappointing results concerning the IPH operators from the literature.
In this paper we have proposed a new IPH based on Hamming distance and exponential decay.
We have shown its effectiveness in isolation for unimodal functions compared to static hypermutations in the ideal conditions when the optimum is known.
Furthermore, we have provided a symmetric version of the operator for the complete Opt-IA AIS to be used in practical applications where the optimum is usually unknown.
We have proved its efficiency for two well-studied bimodal functions.
Future work should evaluate the performance of the proposed algorithm for combinatorial optimisation problems with practical applications.

%
\smallskip
\textbf{Acknowledgments:}
The research leading to these results has received funding from the EPSRC under 
grant agreement no \\ EP/M004252/1.

\bibliographystyle{ACM-Reference-Format}
\bibliography{mybib2} 

\appendix 

\section{Appendix}
\textbf{Proof of Theorem~\ref{th:expoF(x)-UPonOM} } 
\begin{proof}
By Chernoff bounds, the initialised solution has at most $n/2+\epsilon n$ 0-bits w.o.p for any arbitrary small $\epsilon=\Theta(1)$. The probability of improvement in the first mutation step is at least $i/n$ with $i$ showing the number of 0-bits. As we need at most $n/2+\epsilon n$ improvements and each time the mutation fails to make an improvement at least $n^{1-\frac{n-i}{n}}$ fitness function evaluation would get wasted, the total expected time to find the optimum will be $E(T) \leq \sum_{i=1}^{n/2+\epsilon n} \frac{n}{i} \cdot n^{\frac{i}{n}} \leq n^{3/2+\epsilon} \log n$.
To prove the lower bound, consider $i$ as the number of 0-bits. By Chernoff 
bounds, $i$ is at least $n/2-\epsilon n$ 0-bits in the initialised bit string. 
The number of wasted fitness function evaluations at each failure is 
$n^{\frac{i}{n}}$. If we consider the time spent between levels $n(1/2 + 
\epsilon/2)$ and $n(1/2+\epsilon)$, we get the expected time of 
$n^{\frac{n/2-\epsilon n}{n}} \cdot \sum_{i=n/2+n\epsilon/2 }^{n/2+\epsilon 
n} \Omega(1)=n^{1/2-\epsilon} \cdot \epsilon n/2 \cdot \Omega(1)= 
\Omega(n^{3/2-\epsilon})$.

\end{proof}

\textbf{Proof of Theorem~\ref{thm:expohdLO}}
\begin{proof}
The proof is similar to the proof of Theorem \ref{th:linHD-LO} however each 
failure in improvement yields $n^{HD/n}$ wasted fitness function evaluations. 
The expected time to find the optimum is at most $\sum_{i=1}^{n} n \cdot 
n^{HD/n}$ with $i$ showing the number of leading 1-bits. Knowing that the bits 
after the leading ones are uniformly distributed, by Chernoff bounds the number 
of 0-bits (Hamming distance to the $opt$) is less than $1/2 (n-i)+\epsilon n$ 
w.o.p. Hence, the expected time to optimise \textsc{LeadingZeros} is $\sum_{i=1}^{n} n \cdot 
n^{\frac{1/2(n-i)+\epsilon n}{n}}= n \sum_{i=1}^{n} 
n^{\frac{1}{2}-\frac{i}{2n}+\epsilon}=n \cdot n^{1/2+\epsilon} \sum_{i=1}^{n} 
n^{-\frac{i}{2n}}$. Knowing that 
$\sum_{i=0}^{\infty}n^{-\frac{i}{2n}}=\frac{1}{1-n^{(-1/n)}}$ and $n^{-1/n} \leq 
\left(1-\frac{\ln n}{n}\right)(1-o(1))$ (Lemma~\ref{lem:expo}) 
, we get 
$1-n^{(-1/n)} \leq \frac{n}{\ln n}(1+o(1))$. Hence, the expected time is $E(T) 
\leq n^{3/2+\epsilon} \cdot O(\frac{n}{\ln n})=O(\frac{n^{5/2+\epsilon}}{\ln 
n})$.

The proof for lower bound is similar to the proof of Theorem \ref{th:linHD-LO}. 
Here we have $E(f_i) = \frac{i}{n} \left( n^{HD/n}+ E(f_i)\right) + \frac{1}{n} 
\cdot 1 +\frac{n-i-1}{n} \left( 1+ E(f_i)\right)$. Replacing $HD$ with 
$(n-i)/2-\epsilon n$ and then solving this equation for $E(f_i)$ gives us 
$E(f_i)=i n^{1/2-\epsilon-i/(2n)}+n-i$. Then, we get the expected runtime of  
\begin{align*}
&(1-2^{-n/2})\sum_{i=n/2}^{n} E(f_i)\geq \Omega(1) \sum_{i=n/2-\epsilon 
n}^{n/2-\epsilon n/2} 
\left( in^{1/2-\epsilon-i/(2n)}+n-i)\right)\\ &\geq \Omega(1) 
\sum_{i=n/2-\epsilon 
n}^{n/2-\epsilon n/2} 
\left( in^{1/4-\epsilon}\right)= \Omega(n^{9/4-\epsilon}).
\end{align*}
\end{proof}

\end{document}